\documentclass[11pt]{article}
\usepackage[utf8]{inputenc}
\usepackage{verbatim}

\usepackage{one_inc}

\title{The One-Inclusion Graph Algorithm is not Always Optimal} 
\author{Ishaq Aden-Ali\thanks{Department of Electrical Engineering and Computer Science, UC Berkeley. Email: {adenali@berkeley.edu, yeshwanth@berkeley.edu, shetty@berkeley.edu}} \and Yeshwanth Cherapanamjeri\footnotemark[1] \and Abhishek Shetty\footnotemark[1] \and Nikita Zhivotovskiy\thanks{Department of Statistics, UC Berkeley. Email: zhivotovskiy@berkeley.edu}
}
\date{\today}

\begin{document}

\maketitle
\begin{abstract}
The one-inclusion graph algorithm of Haussler, Littlestone, and Warmuth achieves an optimal in-expectation risk bound in the standard PAC classification setup. 
In one of the first COLT open problems, Warmuth conjectured that this prediction strategy \emph{always} implies an optimal high probability bound on the risk, and hence is also an optimal PAC algorithm. 
We refute this conjecture in the strongest sense: for any practically interesting Vapnik-Chervonenkis class, we provide an in-expectation optimal one-inclusion graph algorithm whose high probability risk bound cannot go beyond that implied by Markov's inequality. 
Our construction of these poorly performing one-inclusion graph algorithms uses Varshamov-Tenengolts error correcting codes.  

Our negative result has several implications. 
First, it shows that the same poor high-probability performance is inherited by several recent prediction strategies based on generalizations of the one-inclusion graph algorithm. 
Second, our analysis shows yet another statistical problem that enjoys an estimator that is provably optimal in expectation via a leave-one-out argument, but fails in the high-probability regime. 
This discrepancy occurs despite the boundedness of the binary loss for which arguments based on concentration inequalities often provide sharp high probability risk bounds.
\end{abstract}

\section{Introduction and main result}
Learning binary classifiers is arguably the oldest problem in the theory of machine learning. 
The model, which is captured by the PAC (Probably Approximately Correct) model of learning, traces back to the early works of Vapnik and Chervonenkis \cite{vapnik1964class, vapnik74theory} and of Valiant \cite{valiant1984theory}. A PAC learning algorithm is one that satisfies the following. Given a hypothesis class $\mathcal F$ (a hypothesis is a binary-valued function) defined on a domain $\mathcal X$, we observe a sample of points $(X_1, f^{\star}(X_1)), \ldots, (X_n, f^{\star}(X_n))$, called the \emph{training sample}, where $X_1, \ldots X_n$ are independent copies of a random variable $X \in \mathcal X$ distributed according to an unknown distribution $P$, and $f^{\star} \in \mathcal F$ is an unknown target hypothesis. The aim is to provide an algorithm that, based on the sample, outputs a hypothesis $\widehat{f}$ whose prediction error is as small as possible with high probability with respect to the realization of the training sample. 
Formally, we define the error $\err{\widehat{f}}{P} = \Pr_{X \sim P}\left[\widehat{f}(X) \not= f^\star(X)\right]$ as the probability of missclassification on a freshly sampled point. Since the target hypothesis $f^{\star}$ is in the hypothesis class $\mathcal F$, the most natural strategy is to pick \emph{any} hypothesis $\widehat f \in \mathcal F$ that is \emph{sample-consistent}: $\widehat f(X_i) = f^{\star}(X_i)$ for all $i = 1, \ldots, n$. The hypothesis selected using this strategy is usually referred to as an \emph{empirical risk minimizer} (ERM). Denoting any of these hypotheses by $\widehat{f}_{\textrm{ERM}}$, the standard bound \cite{vapnik1968algorithms, blumer1989learnability} shows that
\begin{equation}
\label{eq:ermbound}
\err{\widehat{f}_{\textrm{ERM}}}{P} = O\left(\frac{d}{n}\log\left(\frac{n}{d}\right) + \frac{1}{n}\log\left(\frac{1}{\delta}\right)\right),
\end{equation}
with probability at least $1 - \delta$, where $d$ is the Vapnik-Chervonenkis (VC) dimension of $\mathcal F$ (formally defined in~\cref{sec:prelim}).
Although this bound has been recently sharpened for ERM for some specific concept classes \cite{hanneke2016refined, zhivotovskiy2018localization}, it is known that the optimal risk bound
\begin{equation}
\label{eq:optimalpac}
\Theta\left(\frac{d}{n} + \frac{1}{n}\log\left(\frac{1}{\delta}\right)\right)
\end{equation}
can only be achieved in general by \emph{improper algorithms} \cite{bousquet2020proper}. These are algorithms that output a hypothesis outside the hypothesis class $\mathcal F$ and thus exclude the standard ERM strategy. The question of achieving the optimal sample complexity \eqref{eq:optimalpac} of PAC learning has been resolved by Hanneke \cite{hanneke2016optimal}, whose solution sharpens the majority vote analysis of Simon \cite{simon2015almost}. The solution uses a recursive majority vote scheme to achieve the optimal sample complexity \eqref{eq:optimalpac}.

However, for many years, the most natural candidate for being an optimal PAC learner was the \emph{one-inclusion graph} algorithm of Haussler, Littlestone and Warmuth \cite{haussler1994predicting},
whose \emph{in-expectation} risk is known to be optimal in this model. We now briefly describe the one-inclusion graph algorithm.
Recall that the projection of a hypothesis class $\mc{F}$ onto a subset of the domain $S=\{x_1, \dots , x_n \} \subseteq \mc{X}$ is defined as $\mc{F}|_S = \{ (f(x_1), \dots , f(x_n)) : f \in \mc{F} \}$.
Informally, the one-inclusion graph algorithm pre-determines a strategy that, for all possible realizations $S$ of the set of unique elements in the training sample and candidate test points $x$, orients/directs the edges of the \emph{one-inclusion graph} $\mc{G}(\mc{F}|_{S\cup \{x\}}) = (V,E)$ whose vertices are $\mc{F}|_{S\cup \{x\}}$ and whose edges connect two hypotheses that only differ on a \emph{single} point in $S \cup \{x\}$. At test time, given a concrete realization of $S$ and test point $x$, the algorithm
finds the edge (a pair of hypotheses) consistent with the training sample in the one-inclusion graph $\mc{G} (\mc{F}|_{S \cup \{x\}})$, and predicts on $x$ using the label that the head of this edge (a hypothesis) assigns to $x$. When the maximum out-degree of the vertices is small, one can show that the expected error of this prediction strategy is small~\cite{haussler1994predicting}. Throughout this section we denote the output of a one-inclusion graph algorithm by $\OIGsimpl$. We postpone a formal description of this algorithm to \cref{sec:prelim}.


The result in \cite{haussler1994predicting} shows that the following error bounds hold for \emph{any} distribution $P$ and \emph{any} target concept $f^{*} \in \mathcal F$,
\begin{equation}
\label{eq:errorboundforoig}
\E\ \err{\OIGsimpl}{P} \le \frac{d}{n + 1}, \quad \textrm{and by Markov's inequality:} \quad \err{\OIGsimpl}{P} \le \frac{d}{(n + 1)\delta},
\end{equation}
with probability at least $1 - \delta$. In the bound above, the expectation is taken with respect to the random sample $X_1, \ldots, X_n$. In one of the first COLT open problems Warmuth conjectured \cite{warmuth2004optimal} that the one-inclusion graph algorithm can \emph{always} achieve the optimal sample complexity \eqref{eq:optimalpac}. One of the motivations of Warmuth was that the one-inclusion graph algorithm is almost optimal in expectation \cite{li2001one}, including even the constant factor which is asymptotically tight. In this context, it is worth mentioning that the existing solution of Hanneke \cite{hanneke2016optimal} has a relatively large constant in the risk bound. Therefore, as noticed in \cite{hanneke2016optimal}, a positive solution to the conjecture of Warmuth could also lead to improved constant factors in the optimal bound \eqref{eq:optimalpac}.
Recently, \cite{larsen2022bagging} used arguments inspired by \cite{hanneke2016optimal} to show that the classical bagging algorithm also achieves the optimal PAC complexity bound, while also bringing up the question about optimal constants.
Some additional interest in this question arises from numerous recent generalizations of the one-inclusion graph algorithm, some of which will be mentioned in~\cref{sec:discussion}. By understanding the basic PAC learning setting, we can possibly improve the bounds for these newer generalizations. 

In this paper, we refute the conjecture of Warmuth in a strong sense by showing that the above application of Markov's inequality is essentially the best one can hope for in general. Our negative result works for almost any practically interesting hypothesis class including the class induced by half-spaces in $\mathbb{R}^p$ for $p \ge 2$. Before stating our result and Warmuth's question formally, we first introduce a few notions.

A hypothesis class $\mc{F}$ is said to contain a \emph{star set} of arbitrary size if for each $n$, there is a subset $S =\{ x_1, \dots , x_n \}$ of the domain and classifiers $f_0 , f_1,\dots , f_n \in \mc{F}$ such that for all $i$, $f_0$ and $f_i$ disagree on a unique point in $S$. 
A more formal definition can be found in~\cref{def:star}.
In particular, such classes are exactly the classes whose \emph{star number} (see the formal definition in \cite[Definition 2]{hanneke2015minimax}) is infinite. 
Many practically interesting VC classes have arbitrarily large star sets. 
These examples include the class of intervals on the real line as well as the class of half-spaces in $\mathbb{R}^p$ for $p \ge 2$. 
We refer to \cite[Section 4.1]{hanneke2015minimax} for a detailed discussion. 

Let $\mathcal F$ be a hypothesis class defined on a countable set $\mathcal X$ with finite VC dimension $d$. 
Informally, we say that a one-inclusion graph algorithm is \emph{valid} for $\mathcal F$ if for any $n$-element subset $S \subset \mathcal X$, the one-inclusion graph of the projection of $\mathcal F$ on $S$ is oriented in a way such that the maximum out-degree of the one-inclusion graph is at most $O(d)$.
A formal definition can be found in~\cref{sec:prelim}.
The result of Haussler, Littlestone, and Warmuth \cite{haussler1994predicting} implies that every VC class has a valid one-inclusion graph algorithm. 
Moreover, for any such algorithm both of the bounds in \eqref{eq:errorboundforoig} hold up to multiplicative constant factors. 
Observe that orientations of one-inclusion graphs with out-degree at most $O(d)$ are not necessarily unique, so in general we have a family of valid strategies. 
Using this definition, we can phrase Warmuth's conjecture~\cite{warmuth2004optimal} as follows.
\begin{conjecture*}[Warmuth \cite{warmuth2004optimal}]
Does the optimal sample complexity bound \eqref{eq:optimalpac} hold for any valid one-inclusion graph algorithm?
\end{conjecture*}

Our main result negatively resolves Warmuth's conjecture.
\begin{theorem}
\label{thm:ourlowerbound}
Let $\mathcal F$ be a hypothesis class defined on a countable set $\mathcal X$ with finite VC dimension $d$. Assume that $\mathcal F$ has a star set of arbitrary size.
There are positive absolute constants $c_1, c_2$ and $c_3$ such that for any sample size $n$ and confidence parameter $\delta$ satisfying $\delta \in ( c_1 d / n, c_2)$, there is a valid one-inclusion graph algorithm that outputs the hypothesis $\widehat{f}_{\operatorname{OIG}}$ such that for some distribution $P = P(n, \delta)$ over $\mc{X}$ and target concept $f^{\star} \in \mathcal{F}$ we have, with probability at least $1 - \delta$, 
\[
\err{\OIGsimpl}{P} \geq c_3 \frac{d}{ n \delta}.
\]
\end{theorem} 

\subsection{Proof overview}
We now provide a high-level overview of our proof. 
For simplicity, we will sketch the argument for the hypothesis class $\mc{F}^{\text{ind}}$ consisting of functions that take on the value $1$ on at most one point in the domain $\mc{X} = \mb{Z}$.
In other words, $\mc{F}^{\text{ind}}$ consists of the zero function and the indicator functions of single points on $\mc{X}$.
It is easy to verify that the VC dimension of $\mc{F}^{\text{ind}}$ is $1$. 
Furthermore, notice that every $n$-element subset $S = \{x_1, \dots, x_n\} \subset \mc{X}$ is a star set for $\mc{F}$, witnessed by $\mc{F}|_{S} = \{f_0, \dots, f_{n}\}$, where $f_0$ is the zero function and $f_i$ is the indicator on $x_i$ for $i \ge 1$.

Recall that the one-inclusion graph algorithm uses the entire training sample $S$ together with the test point $X_{n+1}$ to determine the label of $X_{n+1}$ by building and orienting a one-inclusion graph. 
For our specific class $\mc{F}^{\text{ind}}$, whenever the test point $X_{n+1}$ is not in the training sample, this strategy always boils down to picking between the label $y_{n+1} = 1$ determined by $f_{n+1}$ and the label $y_{n+1} =  0$ determined by $f_0$.
Our goal will be to show that there is a choice of orientation that is valid for this class, but leads to a one-inclusion graph algorithm that has constant error with decent probability.

We now explain the desired properties that we want the orientations we choose to satisfy.
Fix a sample size $n$. 
Our hard distribution will be the uniform distribution on the set $[2n]$ and we will pick the target function $f^\star$ to be the zero function.
Let $S$ be a training sample of size $n$ and assume for simplicity we always sample $n$ unique points.  
Our goal will be to show that a $\Theta(1/n)$ fraction of the possible training samples can be made into ``bad'' training samples.
Specifically, we will show that for any such ``bad'' training sample $S$, when we project $\mc{F}^{\text{ind}}$ onto $S$ and an unseen test point $X_{n+1}$, the orientation determined by the one-inclusion graph algorithm directs the edge $e = \{f_0, f_{n+1}\}$ from the zero function $f_0$ to $f_{n+1}$ for a constant fraction of the realizations of $X_{n+1}$. (See~\cref{fig:OIG-orientations}.)
This would immediately imply that with probability at least $\Theta(1/n)$ (we sample the training sample uniformly) we get a one-inclusion graph algorithm that has constant error.

How should we select these ``bad'' training samples? A naive approach would be to try and pick random orientations, i.e., for each of the possible one-inclusion graphs formed by $n+1$ sized sets $S' \subset [2n]$, pick a random edge $\{f_0, f_1\}$ and direct it towards $f_i$.
Unfortunately, such an approach cannot work since it does not ``coordinate'' the errors well.
Concretely, let $S$ be an $n$-sized training set, $Z = [2n] \setminus S$ (the test points not observed in the training set) and for any $X_{n + 1} \in Z$, $A_{X_{n + 1}}$ is the event that the oriented one-inclusion graph for $S \cup \{X_{n + 1}\}$ directs $e = \{f_0, f_{n + 1}\}$ towards $f_{n + 1}$. Then the probability that more than $t = C \log (n)$, for some large constant $C$, possible extensions of $S$ are oriented towards $f_{n + 1}$ is bounded as
\begin{equation*}
    \Pr \lsrs{\text{At least } t \text{ of the } A_{X_{n + 1}} \text{ occur}} \leq \binom{n}{t} \lprp{\frac{1}{n + 1}}^{t} \leq \frac{n^t}{t!} \cdot \frac{1}{(n + 1)^t} < \frac{1}{t!} < \frac{1}{n}.
\end{equation*}
Hence, this strategy of randomized orientations would result in a set of orientations where at most $1 / n$ fraction of the training samples incur error more than $t / n = O(\log (n) / n)$, a bound significantly worse than the desired bound from \cref{thm:ourlowerbound} which corresponds to constant error in this scenario. The key point of failure in this approach is that for any fixed training sample, $S$, the orientation of each of its extensions is independently chosen. To overcome this, we use a different approach.

We instead shift our perspective towards the possible \emph{training sets}, $S$, that one may observe and correlate the orientations of the $(n + 1)$-sized extensions such that large error is incurred when $S$ is drawn as a training set. Recall that these orientations were previously chosen randomly. Formally, let $\mc{S}$ denote the set of $(2n)$-length binary vectors with exactly $n$-ones where each element corresponds to a possible training set that we may sample. Our goal now is to identify a subset $T$ such that $\abs{T} \geq \Omega (\abs{\mc{S}} / n)$ and for each $S \in T$, orient most of its extensions such that the edge $e = \{f_0, f_{n + 1}\}$ is oriented towards $f_{n + 1}$. The core challenge here is that any given $n + 1$ sized set, $S'$, has $n + 1$ possible training sets which could have been extended to generate it and we need to ensure that while defining extensions for two distinct $S_1, S_2 \in T$, we do not generate two contradictory orientations for the \emph{same} $n + 1$ sized extension, $S'$. Our problem now reduces to the task finding a suitable set, $T$, which simultaneously constitutes a significant fraction of $\mc{S}$ and whose elements do not result in clashes when extended to their $n + 1$ sized counterparts. To do this, we exploit a connection to coding theory.

Consider the following family of vectors:
\begin{align*}
    \code{a}{2n} = \left\{ C \in \{0, 1\}^{2n} : \sum_i^{2n} i \cdot C(i) \equiv a \pmod{2n+1} \right\},
\end{align*}
where $a \in \{0, 1, \dots, 2n\}$.
These are the celebrated \emph{Varshamov-Tenengolts} (VT) error correcting codes (with parameter $a$) introduced in \cite{varshamov1965}.
In the context of coding theory, VT codes are able to recover a transmitted message $C \in \code{a}{2n}$ from a corrupted message $\widetilde{C}$ where a single element in the vector was flipped from a $1$ to a $0$.\footnote{In fact,  as shown by Levenshtein \cite{levenshtein1966binary}, VT codes can even handle single bit deletions. See \cite{Sloanebookchapter}.}
A consequence of this property of $\code{a}{2n}$ is that, for any two vectors $C_1, C_2 \in \code{a}{2n}$ that have the same number of $1$s, $C_1$ and $C_2$ must differ on more than $2$ entries.
We can use this ``uniqueness'' property in the following way: consider the subsets of $\code{a}{2n}$ that have an equal number of 0s and 1s and call this set $T_a$. 
We will view $T_a$ as a collection of possible training samples we can receive. Notice that the extension of any two training samples $S_1, S_2 \in T_a$ to $n+1$ sized sets $S_1'$ and $S_2'$ can never yield the same one-inclusion graph since $S_1' \not= S_2'$ by the ``uniqueness'' property and hence, this rules out the possibility of obtaining contradictory orientations for the same $n + 1$ sized extensions from two different training sets. 
Thus, for any training sample $S \in T_a$, we can coordinate the error of the algorithm by orienting the one-inclusion graphs formed from every extension $S' = S \cup \{x_{n+1}\}$ (where $x_{n+1} \in [2n] \setminus S$) to direct the edge $\{f_0, f_{n+1}\}$ towards $f_{n+1}$.
By picking an $a$ that maximizes the size of $T_a$, we can conclude that $T_a$ contains at least a $1/(2n+1)$ fraction of the possible training samples, each of which is ``bad'' since they induce a prediction error of $1/2$.
Our argument easily applies to VC classes that have arbitrarily large star sets.
In the actual proof we extend the argument above to use multiple $T_a$'s together with a careful application of the probabilistic method. Doing this introduces a tradeoff between prediction error and failure probability that incorporates all three parameters $\delta$, $n$, and $d$.
The full details of our construction can be found in~\cref{sec:construction}.

The remainder of this paper is organized as follows. 
In \cref{sec:discussion} we survey the relevant literature and discuss related results.
In \cref{sec:prelim} we state some preliminary definitions and introduce notation used throughout the paper. 
\cref{sec:construction} is devoted to our construction and the full proof of~\cref{thm:ourlowerbound}.

\begin{figure}[t]
    \centering
    \begin{subfigure}{0.7\textwidth}
    \bigskip
    \resizebox{\textwidth}{!}
{

\tikzset{every picture/.style={line width=0.75pt}} 

\begin{tikzpicture}[x=0.75pt,y=0.75pt,yscale=-1,xscale=1]

\draw   (120,75.5) .. controls (120,65.84) and (127.84,58) .. (137.5,58) .. controls (147.16,58) and (155,65.84) .. (155,75.5) .. controls (155,85.16) and (147.16,93) .. (137.5,93) .. controls (127.84,93) and (120,85.16) .. (120,75.5) -- cycle ;
\draw   (220,75.5) .. controls (220,65.84) and (227.84,58) .. (237.5,58) .. controls (247.16,58) and (255,65.84) .. (255,75.5) .. controls (255,85.16) and (247.16,93) .. (237.5,93) .. controls (227.84,93) and (220,85.16) .. (220,75.5) -- cycle ;
\draw   (320,75.5) .. controls (320,65.84) and (327.84,58) .. (337.5,58) .. controls (347.16,58) and (355,65.84) .. (355,75.5) .. controls (355,85.16) and (347.16,93) .. (337.5,93) .. controls (327.84,93) and (320,85.16) .. (320,75.5) -- cycle ;
\draw  [fill={rgb, 255:red, 0; green, 0; blue, 0 }  ,fill opacity=1 ] (276.5,73) .. controls (276.5,72.45) and (276.95,72) .. (277.5,72) .. controls (278.05,72) and (278.5,72.45) .. (278.5,73) .. controls (278.5,73.55) and (278.05,74) .. (277.5,74) .. controls (276.95,74) and (276.5,73.55) .. (276.5,73) -- cycle ;
\draw  [fill={rgb, 255:red, 0; green, 0; blue, 0 }  ,fill opacity=1 ] (286.5,73) .. controls (286.5,72.45) and (286.95,72) .. (287.5,72) .. controls (288.05,72) and (288.5,72.45) .. (288.5,73) .. controls (288.5,73.55) and (288.05,74) .. (287.5,74) .. controls (286.95,74) and (286.5,73.55) .. (286.5,73) -- cycle ;
\draw  [fill={rgb, 255:red, 0; green, 0; blue, 0 }  ,fill opacity=1 ] (296.5,73) .. controls (296.5,72.45) and (296.95,72) .. (297.5,72) .. controls (298.05,72) and (298.5,72.45) .. (298.5,73) .. controls (298.5,73.55) and (298.05,74) .. (297.5,74) .. controls (296.95,74) and (296.5,73.55) .. (296.5,73) -- cycle ;
\draw    (146.95,89.95) -- (268.99,135.88) ;
\draw [shift={(271.8,136.93)}, rotate = 200.62] [fill={rgb, 255:red, 0; green, 0; blue, 0 }  ][line width=0.08]  [draw opacity=0] (8.93,-4.29) -- (0,0) -- (8.93,4.29) -- cycle    ;
\draw    (245.95,90.95) -- (279.96,126.34) ;
\draw [shift={(282.04,128.5)}, rotate = 226.14] [fill={rgb, 255:red, 0; green, 0; blue, 0 }  ][line width=0.08]  [draw opacity=0] (8.93,-4.29) -- (0,0) -- (8.93,4.29) -- cycle    ;
\draw    (329.05,90.95) -- (295.04,126.34) ;
\draw [shift={(292.96,128.5)}, rotate = 313.86] [fill={rgb, 255:red, 0; green, 0; blue, 0 }  ][line width=0.08]  [draw opacity=0] (8.93,-4.29) -- (0,0) -- (8.93,4.29) -- cycle    ;
\draw   (270,145.5) .. controls (270,135.84) and (277.84,128) .. (287.5,128) .. controls (297.16,128) and (305,135.84) .. (305,145.5) .. controls (305,155.16) and (297.16,163) .. (287.5,163) .. controls (277.84,163) and (270,155.16) .. (270,145.5) -- cycle ;
\draw   (420,75.5) .. controls (420,65.84) and (427.84,58) .. (437.5,58) .. controls (447.16,58) and (455,65.84) .. (455,75.5) .. controls (455,85.16) and (447.16,93) .. (437.5,93) .. controls (427.84,93) and (420,85.16) .. (420,75.5) -- cycle ;
\draw [color={rgb, 255:red, 208; green, 2; blue, 27 }  ,draw opacity=1 ]   (303.27,136.93) -- (424.74,91.42) ;
\draw [shift={(427.55,90.36)}, rotate = 159.46] [fill={rgb, 255:red, 208; green, 2; blue, 27 }  ,fill opacity=1 ][line width=0.08]  [draw opacity=0] (8.93,-4.29) -- (0,0) -- (8.93,4.29) -- cycle    ;

\draw (130,68.9) node [anchor=north west][inner sep=0.75pt]  [font=\scriptsize]  {$f_{1}$};
\draw (230,68.9) node [anchor=north west][inner sep=0.75pt]  [font=\scriptsize]  {$f_{2}$};
\draw (330,68.9) node [anchor=north west][inner sep=0.75pt]  [font=\scriptsize]  {$f_{n}$};
\draw (425,68.9) node [anchor=north west][inner sep=0.75pt]  [font=\scriptsize]  {$f_{n+1}$};
\draw (280,138.9) node [anchor=north west][inner sep=0.75pt]  [font=\scriptsize]  {$f_{0}$};

\end{tikzpicture}

}
    \caption{An orientation that incorrectly predicts using $f_{n+1}$.}
    \end{subfigure}
    \
    \begin{subfigure}{0.7\textwidth}
    \resizebox{\textwidth}{!}
{

\tikzset{every picture/.style={line width=0.75pt}} 

\begin{tikzpicture}[x=0.75pt,y=0.75pt,yscale=-1,xscale=1]

\draw   (120,75.5) .. controls (120,65.84) and (127.84,58) .. (137.5,58) .. controls (147.16,58) and (155,65.84) .. (155,75.5) .. controls (155,85.16) and (147.16,93) .. (137.5,93) .. controls (127.84,93) and (120,85.16) .. (120,75.5) -- cycle ;
\draw   (220,75.5) .. controls (220,65.84) and (227.84,58) .. (237.5,58) .. controls (247.16,58) and (255,65.84) .. (255,75.5) .. controls (255,85.16) and (247.16,93) .. (237.5,93) .. controls (227.84,93) and (220,85.16) .. (220,75.5) -- cycle ;
\draw   (320,75.5) .. controls (320,65.84) and (327.84,58) .. (337.5,58) .. controls (347.16,58) and (355,65.84) .. (355,75.5) .. controls (355,85.16) and (347.16,93) .. (337.5,93) .. controls (327.84,93) and (320,85.16) .. (320,75.5) -- cycle ;
\draw  [fill={rgb, 255:red, 0; green, 0; blue, 0 }  ,fill opacity=1 ] (276.5,73) .. controls (276.5,72.45) and (276.95,72) .. (277.5,72) .. controls (278.05,72) and (278.5,72.45) .. (278.5,73) .. controls (278.5,73.55) and (278.05,74) .. (277.5,74) .. controls (276.95,74) and (276.5,73.55) .. (276.5,73) -- cycle ;
\draw  [fill={rgb, 255:red, 0; green, 0; blue, 0 }  ,fill opacity=1 ] (286.5,73) .. controls (286.5,72.45) and (286.95,72) .. (287.5,72) .. controls (288.05,72) and (288.5,72.45) .. (288.5,73) .. controls (288.5,73.55) and (288.05,74) .. (287.5,74) .. controls (286.95,74) and (286.5,73.55) .. (286.5,73) -- cycle ;
\draw  [fill={rgb, 255:red, 0; green, 0; blue, 0 }  ,fill opacity=1 ] (296.5,73) .. controls (296.5,72.45) and (296.95,72) .. (297.5,72) .. controls (298.05,72) and (298.5,72.45) .. (298.5,73) .. controls (298.5,73.55) and (298.05,74) .. (297.5,74) .. controls (296.95,74) and (296.5,73.55) .. (296.5,73) -- cycle ;
\draw    (146.95,89.95) -- (268.99,135.88) ;
\draw [shift={(271.8,136.93)}, rotate = 200.62] [fill={rgb, 255:red, 0; green, 0; blue, 0 }  ][line width=0.08]  [draw opacity=0] (8.93,-4.29) -- (0,0) -- (8.93,4.29) -- cycle    ;
\draw [color={rgb, 255:red, 208; green, 2; blue, 27 }  ,draw opacity=1 ]   (428.05,89.95) -- (306.07,135.88) ;
\draw [shift={(303.27,136.93)}, rotate = 339.37] [fill={rgb, 255:red, 208; green, 2; blue, 27 }  ,fill opacity=1 ][line width=0.08]  [draw opacity=0] (8.93,-4.29) -- (0,0) -- (8.93,4.29) -- cycle    ;
\draw    (245.95,90.95) -- (279.96,126.34) ;
\draw [shift={(282.04,128.5)}, rotate = 226.14] [fill={rgb, 255:red, 0; green, 0; blue, 0 }  ][line width=0.08]  [draw opacity=0] (8.93,-4.29) -- (0,0) -- (8.93,4.29) -- cycle    ;
\draw    (329.05,90.95) -- (295.04,126.34) ;
\draw [shift={(292.96,128.5)}, rotate = 313.86] [fill={rgb, 255:red, 0; green, 0; blue, 0 }  ][line width=0.08]  [draw opacity=0] (8.93,-4.29) -- (0,0) -- (8.93,4.29) -- cycle    ;
\draw   (270,145.5) .. controls (270,135.84) and (277.84,128) .. (287.5,128) .. controls (297.16,128) and (305,135.84) .. (305,145.5) .. controls (305,155.16) and (297.16,163) .. (287.5,163) .. controls (277.84,163) and (270,155.16) .. (270,145.5) -- cycle ;
\draw   (420,75.5) .. controls (420,65.84) and (427.84,58) .. (437.5,58) .. controls (447.16,58) and (455,65.84) .. (455,75.5) .. controls (455,85.16) and (447.16,93) .. (437.5,93) .. controls (427.84,93) and (420,85.16) .. (420,75.5) -- cycle ;

\draw (130,68.9) node [anchor=north west][inner sep=0.75pt]  [font=\scriptsize]  {$f_{1}$};
\draw (230,68.9) node [anchor=north west][inner sep=0.75pt]  [font=\scriptsize]  {$f_{2}$};
\draw (330,68.9) node [anchor=north west][inner sep=0.75pt]  [font=\scriptsize]  {$f_{n}$};
\draw (425,68.9) node [anchor=north west][inner sep=0.75pt]  [font=\scriptsize]  {$f_{n+1}$};
\draw (280,138.9) node [anchor=north west][inner sep=0.75pt]  [font=\scriptsize]  {$f_{0}$};

\end{tikzpicture}

}
    \caption{Orientation that always correctly predicts using $f_{0}$.
    }
    \end{subfigure}
    \caption{Projection of the class of indicators $\mc{F}^{\text{ind}}$ onto a size $n+1$ star set $S$.
    For $i\ge 1$ the function $f_i$ is the function that disagrees with $f_0$ on $x_{i}$.
    The red edge represents which hypothesis the one-inclusion graph algorithm picks between to determine the label of the point $x_{n+1}$.}
    \label{fig:OIG-orientations}
\end{figure}
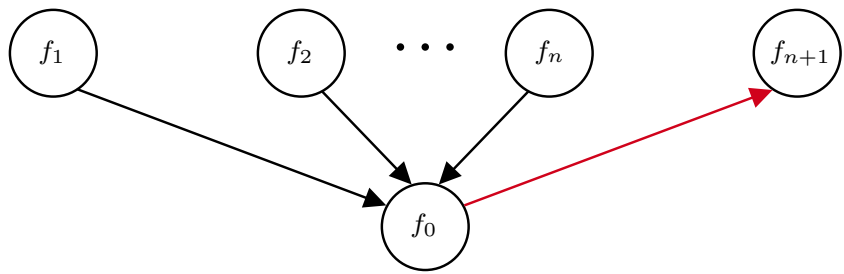
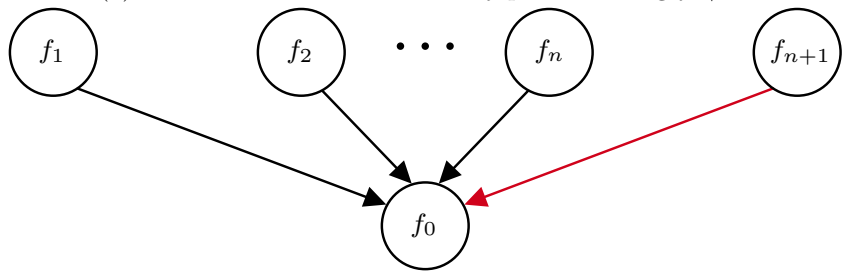

\section{Related work and discussion}
\label{sec:discussion}
In this section, we present some remarks to help the reader place our result in a broader context. We also discuss some relevant literature.

\paragraph{The leave-one-out/exchangeability argument and high probability risk bounds.} The upper bound for any valid one-inclusion graph algorithm is proven in \cite{haussler1994predicting} using a leave-one-out argument. The earliest theoretical analysis of a leave-one-out argument in our context is sometimes attributed to Lunts and Brailovsky \cite{lunts1967evaluation}, who noticed that although this method gives sharp in-expectation risk bounds, it does not necessarily lead to high probability/low variance bounds. 
In fact, Lunts and Brailovsky connected the variance of the leave-one-out bound with stability\footnote{By stability we mean sensitivity to small perturbations in the training sample.} properties of the underlying learning algorithm and provided an example where the variance of the prediction error can be large despite a small leave-one-out error. 
Based on this, Vapnik and Chervonenkis later asked \cite[Chapter VI, Section 7]{vapnik74theory} if, for most practically interesting classification algorithms, a small leave-one-out error also leads to small prediction error with high probability.
It is worth mentioning that the recent analysis of \emph{stable compression schemes} in \cite{Zhivotovskiy2017optimal, bousquet2020proper, hanneke2021stable} shows that a leave-one-out argument can lead to sharp high probability bounds when additional stability-type assumptions are made.
The result of~\cref{thm:ourlowerbound} is on the other end of the spectrum: we construct a valid one-inclusion graph algorithm that does not satisfy any of these stability-type properties. 

\paragraph{Confidence boosting approach.} One of the standard ways to boost low confidence classifiers (corresponding to e.g.,\ $\delta = 1/2$ in \eqref{eq:errorboundforoig}) is based on the following approach. One first splits the sample into approximately $\log(1/\delta)$ non-intersecting equal parts to learn roughly $\log(1/\delta)$ independent classifiers, and then aggregates them to pick the best classifier. 
This idea is exploited in \cite{haussler1994predicting}, where the authors run the one-inclusion graph algorithm multiple times to provide an algorithm whose probability of error is bounded by
\[
O\left(\frac{d}{n}\log\left(\frac{1}{\delta}\right)\right).
\]
While confidence boosting appears to be a general technique to get high probability bounds, we argue that at least in some cases we have to understand the high probability performance of the original algorithm. First, we observe that the term $\log\left(\frac{1}{\delta}\right)$ appears multiplicatively in the risk bound for the boosted algorithm, while the optimal risk bound \eqref{eq:optimalpac} has an additive  $\log\left(\frac{1}{\delta}\right)$ term.
Second, it is not even clear that this approach is generally applicable to leave-one-out based estimators as the confidence boosting approach exploits the realizable/boundedness of the loss assumptions in a strong sense. In particular, it has been recently shown in \cite{mourtada2021} that in a related setup of agnostic linear regression the algorithm of Forster and Warmuth \cite{forster2002relative}, whose analysis follows the same leave-one-out argument used in the analysis of the one-inclusion graph algorithm \cite{haussler1994predicting}, provides a constant risk with at least constant probability. When learning in the presence of noise, Markov's inequality in the confidence boosting trick cannot be applied as the \emph{excess risk} is potentially negative for improper learners (see \cite{mourtada2021} for more details). 

\paragraph{Orientations of one-inclusion graph leading to optimal PAC learners.} For some specific classes the one-inclusion graph algorithm may correspond to an optimal PAC learner. In particular, for \emph{intersection closed} VC classes the so-called \emph{closure} algorithm achieves the bound \eqref{eq:optimalpac}. We refer to \cite{darnstadt2015optimal} for a detailed description as well as to a sequence of papers \cite{helmbold1990learning, auer2007new, hanneke2016refined, bousquet2020proper}, where various results showing the optimality of the closure algorithm are provided. However, as noted by Warmuth \cite{warmuth2004optimal}, the closure algorithm corresponds to a \emph{specific} family of orientations of the one-inclusion graph (the same observation is explicit in \cite[Remark 2]{auer2007new}).
It is not clear if the known arguments (the original proof was based on a network flow argument \cite{haussler1994predicting}, and another existence argument of Haussler \cite{haussler1995sphere} is based on the result of Alon and Tarsi \cite{alon1992colorings}) that focus on the \emph{existence} of valid one-inclusion graph algorithms will lead to orientations exactly corresponding to the closure algorithm for intersection-closed classes. In fact, it follows from our analysis that for some intersection-closed classes we have two valid one-inclusion graph prediction strategies such that one achieves the optimal PAC bound \eqref{eq:optimalpac}, while the second cannot bypass the tail of Markov's inequality \eqref{eq:errorboundforoig}.

\paragraph{Applications and extensions of the OIG algorithm.}

Despite some limitations observed in this paper, the importance of the original  one-inclusion algorithm of Haussler, Littlestone, and Warmuth \cite{haussler1994predicting} follows from various applications and extensions of this algorithm in the literature. Here we list some of them. The analysis of the one-inclusion graph algorithm is used in the proof of Haussler's packing lemma \cite{haussler1995sphere} that provides a sharp bound for covering numbers of VC classes. We refer to \cite{bartlett1998prediction, kupavskii2020epsilon} for extensions and simplifications of Haussler's analysis. Haussler's packing lemma is widely used in empirical process theory and computational geometry. 

A careful inspection of the analysis of Haussler, Littlestone, and Warmuth \cite{haussler1994predicting} shows that the error bound of the one-inclusion graph does not depend on the VC dimension itself, but rather on the average density of subgraphs of the one-inclusion graph.
This observation has allowed for a multiclass extension\footnote{We remark that the analysis of Hanneke's optimal PAC algorithm \cite{hanneke2016optimal} uses a VC dimension based uniform convergence argument so it does not easily apply to these extensions.} of this algorithm leading to new complexity measures and non-trivial risk bounds in this setting \cite{rubinstein2009shifting, simon2010one, daniely2014optimal, brukhim2022characterization}.
Similarly, the one-inclusion graph algorithm can be extended to the (bounded) real valued regression setup \cite{bartlett1998prediction}.
The analysis of Long \cite{long1998complexity} allows, in particular, to extend the one-inclusion graph algorithm to the agnostic classification setting. 
Explicit bounds of this sort can be found in \cite{hanneke2022optimal}. 
Furthermore, the one-inclusion graphs algorithm and some extensions have proven to be useful in other generalizations of the binary classification setup such as universal learning~\cite{bousquet2021theory}, robust learning~\cite{AttiasHM22, montasser2022adversarially}, and several other setups~\cite{HaghifamMRD22, alon2022theory, charikarP2022}.
As we mentioned earlier, our lower bound has implications for some of the settings studied in these papers.
\section{Preliminaries and notation}
\label{sec:prelim}
In our problem setting, there is an unknown probability distribution $P$ over some countable \emph{instance space} $\mc{X}$ that generates examples, and a known hypothesis class $\mc{F}$ which is a collection of binary functions (hypotheses) that map from the instance space $\mc{X}$ to the \emph{label space} $\mc{Y} = \{0,1\}$.
Furthermore, there is an unknown \emph{target} hypothesis $f^{\star} \in \mc{F}$ that labels the examples generated by $P$.
A learning algorithm receives an i.i.d.\ \emph{training sample} $\bar{S} = (X_1, \dots, X_n)$ sampled from $P$ along with their accompanying \emph{labels} $( f^\star(X_1) , \dots , f^\star(X_n) )$ as input, and produces a hypothesis $\widehat{f}$ (not necessarily in $\mc{F}$) as its output.
The goal of the learning algorithm is to produce a hypothesis $\widehat{f}$ that has low error under $P$ which we define to be $\err{\widehat{f}}{P} = \Pr_{X \sim P}\left[\widehat{f}(X) \not= f^\star(X)\right]$.
We will find it convenient to differentiate between the training sample $\bar{S}$ that is potentially a multiset and its corresponding set version $S$ which we will refer to as the \emph{training set}. 
It will also be convenient to define the labelled training sample $(\bar{S}, f^\star(\bar{S})) = ((X_1, f^\star(X_1)) , \dots , (X_n, f^\star(X_n)))$.
The labelled training set  $(S, f^\star(S))$ is defined similarly.

We define the uniform distribution over a finite set $A$ to be $U(A)$.
When the choice of the set $A$ is clear from context we will sometimes abbreviate this to $U$.
For a finite set $A$ and any $k \le |A|$, define $\mc{S}_k(A)$ to be the set of subsets of $A$ of size $k$, i.e.,  $\mc{S}_k(A) = \{A' \subseteq A : |A'| = k\}$.
We will often abbreviate this to $\mc{S}_k$ when the finite set $A$ we use is clear from the context.
For any positive integer $n$ define $[n] = \{1, 2, \dots, n\}$.
Given a hypothesis class $\mc{F}$ and subset of the instance space $S = \{x_1, \dots, x_n\} \subseteq \mc{X}$, we define the projection of $\mc{F}$ onto $S$ to be $\mc{F}|_{S} = \{(f(x_1) , \dots, f(x_n)) : f \in \mc{F} \}$.
In words, the projection is the set of all functions the hypothesis class $\mc{F}$ realizes on the set $S$.
We say an $n$-element set $S = \{x_1, \dots, x_n\} \subseteq \mc{X}$ is \emph{shattered} by the hypothesis class $\mc{F}$ if the projection $ \mc{F}|_{S} = \{0,1\}^n$, i.e., $\mc{F}$ realizes every possible function on $S$.
The \emph{Vapnik Chervonenkis (VC) dimension} of a hypothesis class $\mc{F}$ is the largest integer $d$ such that there exists a $d$-element subset of $\mc{X}$ that is shattered by $\mc{F}$.
Throughout this paper we will use the variable $d$ to denote the VC dimension of hypothesis class $\mc{F}$ and it will always be clear from context which hypothesis class $d$ will correspond to.
For a projection $\mc{F}|_{S}$ with $|S| = n$, we define the Hamming distance between two hypotheses $f$ and $g$ in $\mc{F}|_{S}$, denoted by $\rho_{n}(f,g)$, to be the number of elements in $S$ that $f$ and $g$ differ on: $\rho_{n}(f,g) = |\{x \in S : f(x) \not= g(x) \}|$.

We next define the notion of star sets discussed earlier. 
The notion was first explicitly defined in the context of active learning, see \cite{hanneke2015minimax}.
This will be the main notion of complexity that will be used in our lower bound.

\begin{definition}[Star sets] \label{def:star}
    A concept class $\mathcal F$ defined on an infinite domain $\mathcal X$ has a \emph{star set of arbitrary size} if for any integer $n$, there exist $S = \{x_1, \ldots, x_n\} \subset \mathcal X$ and classifiers $f_{0}, f_1 \ldots, f_n \in \mathcal F$ such that for all $i = 1, \ldots, n$,
    \[
    \{x \in S : f_{i}(x) \neq f_{0}(x)\} = \{x_i\}.
    \]
    For any $n$ such a set $S$ is a \emph{star set}. The corresponding classifiers $\{f_i\}_{i = 0}^n$ \emph{witness} the star set $S$.
    \end{definition}

We now formally define one-inclusion graphs, orientations of one-inclusion graphs, and the one-inclusion graph algorithm.
\begin{definition}[One-inclusion graph]
Fix a hypothesis class $\mc{F}$ and an $n$-element subset of the domain $S \subseteq \mc{X}$.
The one inclusion graph $\mc{G}(\mc{F}|_{S}) = (V,E)$ has its vertex set as $V \coloneqq \mc{F}|_{S}$ and the edge set as
\[
E \coloneqq \{ \{f,g\} : f,g \in \mc{F}|_S \ , \ \rho_{n}(f,g) = 1\}.
\]
In words, we connect an edge between two vertices (projected hypotheses) if and only if they differ on a single element in the set $S$.
\end{definition}

In order to use the one-inclusion graph in a prediction algorithm, we will ``orient'' the undirected one-inclusion graph into a directed graph.
\begin{definition}[Orientations for one-inclusion graphs]
Fix hypothesis class $\mc{F}$ and
any finite subset $S \subseteq \mc{X}$. An \emph{orientation} of $\mc{G}(\mc{F}|_{S})$ is a function $\sigma_{S} : E \to V$ such that for any $e \in E$, $\sigma_{S}(e) \in e$.
Let $V$ be a collection of finite subsets of $\mc{X}$. An \emph{orientation rule for} $V$, usually denoted by $\sigma$, is a collection of orientations for each element of $V$, that is,
\[
\sigma = \{ \sigma_{S} : S \in V\}.
\]
An \emph{orientation rule} now denotes orientations for all possible finite subsets of $\mc{X}$.
\end{definition}
Put differently, an orientation takes an \emph{undirected} one-inclusion graph and defines a corresponding \emph{directed} one-inclusion graph by determining the head of each edge.
An orientation rule for hypothesis class $\mc{F}$ just tells us how we should orient any of the possible one-inclusion graphs we can obtain from $\mc{F}$ and any finite subset of the domain $S$.
For a one-inclusion graph $\mc{G}(\mc{F}|_{S}) = (V,E)$ and orientation $\sigma_S$, define the \emph{out-degree} of a vertex $v \in V$ to be
\[
\out(v;\sigma_{S}) = |\{e \in E : v \in e, \sigma_{S}(e) \not= v\}|.
\]
The \emph{max out-degree} of $\mc{G}(\mc{F}|_{S})$ is naturally defined as $\out(\sigma_{S}) = \max_{v \in V} \out(v;\sigma_{S})$.

We are now ready to describe the one-inclusion graph algorithm.
We present its pseudocode in~\cref{alg:OIG}.
A one-inclusion graph algorithm is defined by the orientation rule $\sigma$ it uses and the labeled training set $(S, f^\star(S))$ it receives.
This is reflected in the notation $\OIG{\sigma}{S}$ that we use for the hypothesis produced by the one-inclusion graph algorithm (see~\cref{alg:OIG}).
We now define what it means for a one-inclusion graph algorithm to be valid.
\begin{definition}
Fix a hypothesis class $\mc{F}$ with VC dimension $d$.
We say a one-inclusion graph algorithm for $\mc{F}$ that uses orientation rule $\sigma$ is \emph{valid} if
\[
\max_{\sigma_{S} \in \sigma} \out(\sigma_{S}) = O(d).
\]
\end{definition}

A beautiful result of Haussler, Littlestone and Warmuth shows that there is always a valid one-inclusion graph algorithm for any hypothesis class with finite VC dimension~\cite{haussler1994predicting}. 
\begin{theorem}{\cite[Theorem~2.2]{haussler1994predicting}}
    \label{thm:hlw_out_degree}
    For any hypothesis class $\mc{F}$ over instance space $\mc{X}$ with VC dimension at most $d$ and any finite $S \subset \mc{X}$, there exists an orientation $\sigma_S$ of $\mc{G} (\mc{F}|_{S})$ with
    \begin{equation*}
        \out (\sigma_S) \leq d.
    \end{equation*}
\end{theorem}
\begin{algorithm}[H]
\caption{One-inclusion graph algorithm.}
\label{alg:OIG}
\textbf{Inputs:} Labelled training set $(S, f^\star(S))$, orientation rule $\sigma$.\\
\textbf{Output:} Hypothesis $\OIG{\sigma}{S} : \mc{X} \to \mc{Y}$.
\vspace{1em}

For any point $x \in \mc{X}$ the hypothesis $\OIG{\sigma}{S}$ predicts as follows:\hspace{-5em}
\vspace{1em}
\begin{algorithmic}[1] 
    \STATE If there is a unique label $y$ for $x$ consistent with $(S,f^\star(S))$ and $\mc{F}$, predict $y$.\hspace{-1em}
    \vspace{0.5em}
    \STATE Let $e$ be the edge in $\mc{G}(\mc{F}|_{S \cup \{x\}})$ with hypotheses consistent with $(S , f^\star(S))$ but not on $x$.
    \vspace{0.5em}
    \STATE Predict according to $ \sigma_{S\cup\{x\}} (e)$, i.e., the hypothesis pointed to in the orientation of $e$.
\end{algorithmic}
\end{algorithm}
\section{The construction}
\label{sec:construction}
In this section, we prove our main result.
Let $A = \{x_1, \dots, x_m\}$ be a set of size $m$ with a fixed ordering of its elements.
For any $k$-element subset $A' \subseteq A$ with $k \le m$, we will slightly overload notation by simultaneously referring to $A'$ as a subset of elements of $A$, and the length $n$ binary vector that is $1$ in the $i$-th entry if and only if $x_i$ is in $A'$.
We will always make it clear what the set $A$ is to make sure this overloaded notation makes sense.

We begin by recalling VT codes and stating a useful property they enjoy.
VT codes of length $m$ and parameter $a \in \{0,1 \dots, m\}$ are given by the following family of vectors:
\begin{align*}
    \code{a}{m} = \left\{ C \in \{0, 1\}^{m} : \sum_{i=1}^{m} i \cdot  C(i) \equiv a \pmod{m+1}  \right\}.
\end{align*}
We will find the following notion of coverage to be useful.
\begin{definition}[Coverage]
    Fix an integer $m$. We say a vector $S' \in \{0,1\}^{m}$ \emph{covers} the vector $S \in \{0,1\}^{m}$ if there is an index $i \in \left[ m \right]$ such that $ S(i) =1 $ but $S'(i) = 0$, and for every $j \neq i$ we have $S(j) = S'(j)$. We will denote this by $ S' \prec S$.
\end{definition}

With this notation, VT codes satisfy a ``uniqueness'' property crucial for our analysis. 
\begin{lemma}[Unique neighborhoods]\label{lem:uniqueness}
Let $A$ be a set of $m$ elements. For any $k \le m$, binary vector $S \in \mathcal{S}_{k}(A)$ and $a \in \{0, 1, \dots, m\}$, there exists at most one binary vector $S' \in \code{a}{m}$ with $ S' \prec S$.
\end{lemma}
\begin{proof}
    Towards a contradiction assume that for some $S \in \mc{S}_k(A)$ there were two vectors $S', S'' \in \code{a}{m} $ that satisfied the property above with corresponding indices $i'$ and $i''$. Then, we have that $S'$ and $S''$ disagree only on $i'$ and $i''$. 
    So 
    \begin{align}
        0 \equiv \sum_{i=1}^m i \cdot S'(i) - \sum_{i=1}^m i \cdot S''(i)  \equiv i'' - i' \pmod{m+1}.
     \end{align}
     This is a contradiction since $1 \leq i', i'' \leq m$. 
\end{proof}

Let $\mc{F}$ be the hypothesis class with star sets of arbitrary size and VC dimension $d$ for which we would like to construct the bad one-inclusion graph algorithm.
Let $n$ denote the sample size and let $ \starset \subset \mc{X}$ be a star set of size $2n$.
Let $f_0, f_1, \dots , f_{2n}$ denote the corresponding functions that witness the star set. 
Through a simple re-labeling procedure, we may assume that the center of the star set $f_0$ is the all zeros function; i.e.,  $f_0(x) = 0$ for all $x \in  \starset$. Our distribution will simply be the uniform distribution over $ \starset $ with the labels generated by $ f^{\star} =  f_0$. 
Hence, our labelled training sample consists of points $\bar{S} = ((X_1, f^\star(X_1)), \dots,(X_n, f^\star(X_n))) $ where each $X_i$ is sampled from the uniform distribution $U := U( \starset )$ and $f^\star(X_i) = 0$. 

We will first prove a lower bound on the error conditioned on the number of unique elements observed in our training sample $\bar{S} \sim U^n$, i.e., the size of the \emph{training set} $S$. 
Let $|S| = k$.
Note that $S$ may be associated with an element of $\mc{S}_k(\starset)$. Slightly abusing notation we denote $\mc{S}_k(\starset)$ by $\mc{S}_k$, and $\mc{S}_{k+1}(\starset)$ by $\mc{S}_{k + 1}$ respectively.
We will now define a bipartite graph with one set of vertices a subset of $\mc{S}_k$ (and hence, a possible realization of $S$) and the other $\mc{S}_{k + 1}$. 
This graph will then be used to construct a orientation rule with poor performance.
Formally, the vertex sets of the graph are defined below
\begin{gather*}  
    V_1 = \bigcup_{0\leq i < 4{\ceil{\delta n}} } T_{(i)}, \\
    V_2 = \mc{S}_{k + 1}, \tag{VERT-SETS} \label{eq:vert-set}
\end{gather*}
where $ T_i \coloneqq \code{i}{2n} \cap \mc{S}_k$ is the intersection of the code $\code{i}{2n}$ with the set of vectors containing exactly $k$ ones and $T_{(0)}, \dots, T_{(2n)}$ is a re-ordering such that $|T_{(0)}| \ge |T_{(1)}| \ge \dots \ge |T_{(2n)}|$.
In particular, since $\code{i}{2n} \cap \code{j}{2n} = \emptyset$ for $i \neq j$, we have for any $\ell \in [2n] \cup \{0\}$,
\begin{equation}
\label{eq:reordering}
\sum\nolimits_{i = 0}^{\ell} |T_{(i)}| \ge \frac{\ell + 1}{2n+1}\cdot|\mc{S}_k| \ge \frac{\ell + 1}{3n}\cdot|\mc{S}_k|.
\end{equation}
We will construct our family of one-inclusion graphs such that for a substantial fraction of the training sets in $V_1$, the corresponding one-inclusion strategy incurs large error. We now construct the edge set of our (undirected) bipartite graph, $G = (V_1, V_2, E)$. Define for all $v \in V_1 \cup V_2$, 
\[
N(v) = \{u \in V_1 \cup V_2: (u, v) \in E\},
\]
that is, the set of neighbors of $v$ in $G$. We will use the probabilistic method to construct an edge set, $E$, satisfying certain cardinality constraints. The first constraint ensures that the family of one-inclusion graphs we construct from $G$ satisfies the appropriate out-degree constraints while the second will be used to show that they incur large error.
\begin{lemma}
    \label{lem:matching_graph}
    There exists an edge set $E$ such that the bipartite graph $G = (V_1, V_2, E)$ satisfies,
    \begin{gather*}
       \text{for all}\ v \in V_2,\  \abs{N(v)} \leq d, \quad \text{and} \quad
        \abs*{\lbrb{v' \in V_1: \abs{N(v')} \geq 
        \frac{d}{8\delta}}} \geq \frac{3}{4}\abs{V_1}.
    \end{gather*}
\end{lemma}
\begin{proof}
Our proof will utilize the probabilistic method. Defining $M(v) = \lbrb{v' \in V_1: v' \prec v}$ for all $v \in V_2$, $E$ is constructed according to the following random process. For any $v \in V_2$:
\begin{enumerate}
    \item With $\ell = \min\left\{d, \abs{M(v)} \right\}$, pick $\{v_{i}'\}_{i \in [\ell]}$ uniformly at random without replacement from $M(v)$.
    \item Add all edges $\lbrb{(v, v_{i}')}_{i \in [\ell]}$ to $E$.
\end{enumerate}
Note that first claim of the lemma follows immediately from the definition of the probabilistic process. For the second, fix $v' \in V_1$ and define
\begin{equation*}
    Q(v') = \{v \in \mc{S}_{k + 1}: v' \prec v\}, 
\end{equation*}
and for $v \in V_2$ define 
\begin{equation*}
    p_{v} = \frac{\min\left\{d, \abs{M(v)}\right\}}{\abs{M(v)}}.
\end{equation*}
We have by the above definition and linearity of expectation
\begin{equation*}
    \E \lsrs{\abs{N(v')}} = \sum_{v \in Q(v')} \frac{\min\left\{d, \abs{M(v)}\right\}}{\abs{M(v)}} = \sum_{v \in Q(v')} p_{v}.
\end{equation*}
From the independence of the indicators $\bm{1} \lbrb{v \in N(v')}$,
\begin{equation*}
    \Var (\abs{N(v')}) = \Var \lprp{\sum_{v \in Q(v')} \bm{1} \lbrb{v \in N(v')}} = \sum_{v \in Q(v')}  \Var \lprp{\bm{1} \lbrb{v \in N(v')}} \leq \sum_{v \in Q(v')} p_{v}.
\end{equation*}
By \cref{lem:uniqueness} and the definition of $V_1$, we have $\abs{M(v)} \leq 4 \ceil{\delta n}$ for all $v \in V_2$. 
When $c_1$ (\cref{thm:ourlowerbound}) is large enough, we have $4\ceil{\delta n} \leq 6\delta n$ and $n\delta \geq c_1 d$, which yields
\begin{equation*}
    p_{v} \geq \frac{1}{6} \cdot \frac{d}{\delta n}.
\end{equation*}
Additionally, we have when $c_2$ (\cref{thm:ourlowerbound}) is small enough and using the fact that $|Q(v')| \ge n$,
\begin{equation*}
    \frac{\sqrt{\Var (\abs{N(v')})}}{\E \lsrs{\abs{N(v')}}} \leq \frac{\sqrt{\sum_{v \in Q(v')} p_v}}{\sum_{v \in Q(v')} p_v} = \sqrt{\frac{1}{\sum_{v \in Q(v')} p_v}} \leq \sqrt{\frac{6\delta n}{d\abs{Q(v')}}} \leq \sqrt{\frac{6\delta}{d}} \leq \frac{1}{10}.
\end{equation*}
This yields via Chebyshev's inequality,
\begin{align*}
     \Pr \lsrs{\abs{N(v')} \geq \frac{d}{8 \delta}} \geq \Pr \lsrs{\abs{N(v')} \geq \frac{\abs{Q(v')}}{8} \cdot \frac{d}{\delta n}} \geq \Pr \lsrs{\abs{N(v')} \geq \frac{3}{4} \sum_{v \in Q(v')} p_v} \geq \frac{3}{4}.
\end{align*}
Linearity of expectation (on $v' \in V_1$) establishes the second claim by the probabilistic method.
\end{proof}

With the bipartite graph, $G = (V_1, V_2, E)$, defined in \ref{eq:vert-set} and \cref{lem:matching_graph}, we will now construct our family of one-inclusion graphs for the sets in $\mc{S}_{k + 1}$.
We think of an element of $S' \in \mc{S}_{k+1}$ as the union of a training set $S$ and a new test point $x_{n+1}$.
Informally, for any set $S' \in \mc{S}_{k + 1}$, we pick any low out-degree orientation and then reorient the edges of the zero function. We prove that for most of the $k$-sized training sets $S \in V_1$ and a significant fraction of their ``extensions'' to $(k + 1)$-sized sets $S' \in V_2$, the one-inclusion graph algorithm predicts incorrectly on $x_{n+1} \in S' \setminus S$. Our construction is formally described in \cref{alg:BadOIG}. Note that \cref{alg:BadOIG} only constructs an orientation rule for $V_2$ (which corresponds to $\mc{S}_{k + 1}$ in this context).

\begin{algorithm}[H]
    \caption{Constructing sub-optimal one-inclusion graphs}
    \label{alg:BadOIG}
    \textbf{Input:} Bipartite Graph $G = (V_1, V_2, E)$. \\
    \textbf{Output:} Orientation rule $\sigma$ for $V_2$.
    \vspace{0.5em}
    \begin{algorithmic}[1]
    \setcounter{ALC@unique}{0}
    \STATE Set $\sigma = \{\}$.
    \vspace{0.5em}
    \FOR{$S' \in V_2$}
        \vspace{0.5em}
        \STATE Pick any orientation $\sigma_{S'}$ that has max out-degree at most $d$.\hspace{-1em}\label{line:aribtrary_orientation}
        \vspace{0.5em}
        \STATE For every $S \in N(S')$, reorient $\sigma_{S'}(\{f_0,  f_{i}\}) = f_i$ where $f_i = \bm{1} \lbrb{x \in S' \setminus S }$.
        \vspace{0.5em}
        \STATE For any other other edge $\{f_i,f_0$\} not modified in the previous step set $\sigma_{S'}(\{f_0,  f_{i}\}) = f_0$.~\label{line:reorientation}
        \vspace{0.5em}
        \STATE Add $\sigma_{S'}$ to $\sigma$.
        \vspace{0.5em}
    \ENDFOR
    \vspace{0.5em}
    \STATE \textbf{return} $\sigma$.
    \end{algorithmic}
\end{algorithm}

Before we proceed, we first show that the orientation rule constructed in \cref{alg:BadOIG} is valid.
\begin{lemma}
    \label{lem:valid_oig}
    When $G$ satisfies the conclusion of \cref{lem:matching_graph} the orientation rule for $V_2$, $\sigma$, constructed in~\cref{alg:BadOIG} satisfies
    \begin{equation*}
        \max_{\sigma_{S'} \in \sigma} \out(\sigma_{S'}) \le d+1.
    \end{equation*}
\end{lemma}
\begin{remark}
    For the sake of generality, we prove \cref{lem:valid_oig} with an upper bound of $d+1$.
    If we consider more structured classes, e.g., the class of functions that take on the value $1$ at most $d$ times, we can improve this bound to $d$.
    This would match the out-degree bound of \cref{thm:hlw_out_degree}.
\end{remark}
\begin{remark}
The bound $d + 1$ on the max out-degree only slightly changes the risk bound for the one-inclusion graph algorithm. In particular, our orientation rule gives
\[
\E\ \err{\OIGsimpl}{P} \le \frac{d + 1}{n + 1}, \quad \textrm{and by Markov's inequality:} \quad \err{\OIGsimpl}{P} \le \frac{d + 1}{(n + 1)\delta}.
\]
\end{remark}
\begin{proof}[Proof of \cref{lem:valid_oig}]
Let $S' \in \mc{S}_{k + 1}$. The out-degree of the all-zeros hypothesis $f_0$ is at most $d$ by the first claim of \cref{lem:matching_graph}. 
For any hypothesis $g \in \mc{F}|_{S'}$ with $ g \neq f_0$, its out-degree is at most $d$ in the original orientation $\sigma_{S'}$ selected in
Line 3 of~\cref{alg:BadOIG} (such an orientation exists by \cref{thm:hlw_out_degree}).
We add at most one outgoing edge from $g$ to $f_0$ in 
Line 5 concluding the proof.
\end{proof}

Our last lemma will show that the one-inclusion strategy defined by the orientation rule $\sigma$ constructed in \cref{alg:BadOIG} incurs large error for most training sets in $V_1$. In fact, this will be shown for the subset whose existence was established in \cref{lem:matching_graph}. Formally, define $W_1^k$ as
\begin{equation*}
    W_1^k = \lbrb{v \in V_1: \abs{N(v)} \geq \frac{d}{8 \delta n}}.
\end{equation*}
The following lemma shows that any training set from $W_1^k$ incurs large error.
\begin{lemma}
\label{lem:bad_dataset_error}
For any $S \in W_1^k$, the one-inclusion graph algorithm that predicts using the orientation rule $\sigma$ defined by \cref{alg:BadOIG} on input $G$ satisfying the conclusion of \cref{lem:matching_graph} has
\begin{equation*}
    \err{\OIG{\sigma}{S}}{P} \geq \frac{1}{16} \cdot \frac{d}{\delta n}.
\end{equation*}
\end{lemma}
\begin{proof}
By the construction of the orientation rule $\sigma$ in \cref{alg:BadOIG}, we only make an error on $x$ if $S \cup \{x\} \in N(S)$, as we only predict $1$ in this scenario. We now have
\begin{equation*}
    \err{\OIG{\sigma}{S}}{P} \geq \frac{\abs{N(S)}}{2n} \geq \frac{1}{16} \cdot \frac{d}{\delta n}. \qedhere
\end{equation*}
\end{proof}
We will now prove \cref{thm:ourlowerbound}. Recall that our distribution is the uniform distribution over the participating elements of a star set of size $2n$ with the center canonically identified with the zero function and the elements with the set $[2n]$. For each $k \in [n]$ denoting the possible number of unique observed elements, consider the family of one-inclusion prediction strategies constructed in \cref{alg:BadOIG} with the bipartite graph defined in \ref{eq:vert-set} and \cref{lem:matching_graph}. By \cref{lem:bad_dataset_error}, any unique training set $S$ in $W^k_1$ incurs large error. Hence, we only need to lower bound the probability of observing a training set from $W_1 \coloneqq \cup_{k \in [n]} W^k_1$. Note that conditioned on $\abs{S} = k$, the training set $S$ is uniformly distributed on $\mc{S}_{k}$.
Hence, we have by \cref{lem:matching_graph} and \eqref{eq:reordering} that
\begin{align*}
    \Pr_{\bar{S} \sim U^n}\left[S \in W^k_1 \large \, \middle| \, \abs{S'} = k\right] &= \frac{\abs{W^k_1}}{\abs{\mc{S}_k}} \geq \frac{3}{4} \cdot \frac{\abs{V_1}}{\abs{\mc{S}_k}} = \frac{3}{4} \cdot \sum_{i = 0}^{4\ceil{\delta n} -1} \frac{\abs{T_{(i)}}}{\abs{\mc{S}_k}} \ge \frac{3}{4|\mc{S}_k|} \cdot \frac{4 \ceil{\delta n} |\mc{S}_k|}{3n} \ge \delta.
\end{align*}
To conclude the proof of the theorem, we now have
\[
    \Pr_{\bar{S} \sim U^n}\left[ S \in W_1 \right]  = \sum_{k=1}^{n}  \Pr_{\bar{S} \sim U^n}\left[S \in W^k_1 \large \, \middle| \, \abs{S'} = k\right] \cdot \Pr_{\bar{S} \sim U^n}\left[ |S| = k \right] \ge \delta.
\]
\qed

\paragraph{Acknowledgments.} The authors would like to thank Omar Alrabiah for fruitful discussions.

{\footnotesize
\bibliographystyle{alpha}
\bibliography{refs.bib}
}

\end{document}